\documentclass[sigconf]{acmart}

\usepackage{hyperref}
\hypersetup{
    colorlinks=true,
    linkcolor=black,
    urlcolor=black,
    citecolor=black,
    breaklinks=true 
}

\def\equationautorefname~#1\null{Equation~(#1)\null}
\usepackage{breakurl}
\usepackage{color}
\usepackage{bm}
\usepackage{xspace}
\usepackage{algorithmic}
\usepackage{algorithm}
\usepackage{mdwlist}    
\usepackage{paralist} 
\usepackage{enumitem}  

\usepackage{utfsym}

\usepackage{adjustbox}
\usepackage{array}
\usepackage{booktabs}
\usepackage{multirow}

\usepackage{array,graphicx}
\usepackage{booktabs}
\usepackage{pifont}
\usepackage{color}

\usepackage{colortbl}
\definecolor{lightgray}{gray}{0.85}
\usepackage{multirow} 
\usepackage{comment}

\usepackage{fancybox} 
\usepackage{amsmath}
\usepackage{amsfonts}

\newcommand{\propose}{{\textsc{KTVGL}}\xspace}
\newcommand{\stream}{{\textsc{KTVGL-Stream}}\xspace}

\newcommand{\mypara}[1]{\vspace{0em}\noindent\textbf{#1.}}

\newcommand{\X}{\mathcal{X}}
\newcommand{\R}{\mathbb{R}}
\newcommand{\Dminus}{D_{\setminus m}}

\newtheorem{lemma}{\textsc{Lemma}}
\newtheorem{problem}{Problem}

\newcommand{\hide}[1]{}

\AtBeginDocument{%
  }

\copyrightyear{2026}
\acmYear{2026}
\setcopyright{cc}
\setcctype{by}
\acmConference[WWW '26]{Proceedings of the ACM Web Conference 2026}{April 13--17, 2026}{Dubai, United Arab Emirates}
\acmBooktitle{Proceedings of the ACM Web Conference 2026 (WWW '26), April 13--17, 2026, Dubai, United Arab Emirates}
\acmPrice{}
\acmDOI{10.1145/3774904.3792608}
\acmISBN{979-8-4007-2307-0/2026/04}

\begin{document}

\title{Interpretable Dynamic Network Modeling of Tensor Time Series via Kronecker Time-Varying Graphical Lasso}

\author{Shingo Higashiguchi}
\affiliation{%
  \institution{SANKEN, The University of Osaka}
  \city{Osaka}
  \country{Japan}
}
\email{shingo88@sanken.osaka-u.ac.jp}

\author{Koki Kawabata}
\affiliation{%
  \institution{SANKEN, The University of Osaka}
  \city{Osaka}
  \country{Japan}
}
\email{koki@sanken.osaka-u.ac.jp}

\author{Yasuko Matsubara}
\affiliation{%
  \institution{SANKEN, The University of Osaka}
  \city{Osaka}
  \country{Japan}
}
\email{yasuko@sanken.osaka-u.ac.jp}

\author{Yasushi Sakurai}
\affiliation{%
  \institution{SANKEN, The University of Osaka}
  \city{Osaka}
  \country{Japan}
}
\email{yasushi@sanken.osaka-u.ac.jp}


\begin{abstract}

With the rapid development of web services, large amounts of time series data are generated and accumulated across various domains such as finance, healthcare, and online platforms.
As such data often co-evolves with multiple variables interacting with each other, estimating the time-varying dependencies between variables (i.e., the dynamic network structure) has become crucial for accurate modeling.
However, real-world data is often represented as tensor time series with multiple modes, resulting in large, entangled networks that are hard to interpret and computationally intensive to estimate.
In this paper, we propose Kronecker Time-Varying Graphical Lasso (KTVGL), a method designed for modeling tensor time series.
Our approach estimates mode-specific dynamic networks in a Kronecker product form, thereby avoiding overly complex entangled structures and producing interpretable modeling results.
Moreover, the partitioned network structure prevents the exponential growth of computational time with data dimension.
In addition, our method can be extended to stream algorithms, making the computational time independent of the sequence length.
Experiments on synthetic data show that the proposed method achieves higher edge estimation accuracy than existing methods while requiring less computation time. 
To further demonstrate its practical value, we also present a case study using real-world data.
Our source code and datasets are available at \url{https://github.com/Higashiguchi-Shingo/KTVGL}.
\end{abstract}



\maketitle

\section{Introduction}
\label{sec:Introduction}

In recent years, the rapid development and proliferation of web services have led to an explosive accumulation of time series data across diverse domains including finance \cite{FinanceGL,Repeat}, public health \cite{epicast}, and online platforms such as social networks \cite{GraphOverTime} and web search logs \cite{CubeCast,FluxCube}. 
Because such data reflects users’ social activities and interests, analyzing it is essential for tasks such as understanding user behavior and formulating marketing strategies. 
For instance, identifying pairs or groups of series that are positively or negatively correlated within multivariate time series can improve recommendation systems and demand forecasting, thereby facilitating decision-making processes.

Graphical-model-based methods, which infer a graph structure (i.e., network) where each node corresponds to a variable and edges represent conditional dependencies, are widely used for learning relationships from multivariate data.
Graphical Lasso \cite{GLasso} estimates a static precision matrix under sparsity constraints to capture such dependencies.
To extend this framework to multivariate time series, it was extended to Time-varying graphical lasso (TVGL) \cite{TVGL}, enabling dynamic network inference.
TVGL is notable for its ability to capture time-varying dependencies, which are common in many real-world time series data, and has inspired several extensions \cite{KernelTVGL,LTVGL}.
However, real-world data is frequently obtained in the form of tensor time series with many attributes.
For example, consider web search volume data obtained from GoogleTrends \cite{googletrend}: the data constitutes a 3-order tensor with attributes (\textit{time, keyword, country; value}).
To apply methods for multivariate time series (e.g., TVGL) to such tensor data, the tensor time series must be flattened into a higher dimensional time series.
Such transformation causes the loss of structural information within the tensor data.
As a result, the model mixes up all relationships among variables, captures spurious correlations, and yields a large, entangled network structure that reduces the interpretability of the modeling results.
Moreover, its computational time increases significantly as the number of variables in a mode increases.
Therefore, a new approach tailored for tensor time series is required.
Figure \ref{fig:figure1} illustrates the difference between our approach and existing methods.
In the right figure, the input tensor time series is flattened to create a multivariate time series before estimating the network. 
This leads to a substantial increase in computation time and reduced interpretability of the modeling results due to the large, entangled network.
On the other hand, our method estimates multiple mode-specific dynamic networks corresponding to each non-temporal mode, as shown in the left figure.
Our model disentangles the complex interactions within the tensor time series, enabling analysis of node interactions and their temporal evolution within each mode. 
Consequently, it provides users with powerful tools for gaining new insights into tensor time series.
Specifically, it enables tasks that were previously unachievable with conventional methods, such as capturing interactions focused solely on specific modes or detecting change points focused exclusively on specific modes.

In this paper, we propose Kronecker Time-Varying Graphical Lasso (\propose), a method for dynamic network inference on tensor time series.
The proposed method estimates multiple networks, each of which is a sparse dependency network of the corresponding non-temporal mode of the input tensor.
This structure avoids estimating a single massive network, thereby reducing computational time and improving the interpretability of modeling results through concise multi-network representations.
Furthermore, we impose a Kronecker structure on these networks.
Specifically, the proposed model consists of mode-specific networks, whose Kronecker product represents the dependency structure of the entire data when regarded as a multivariate time series.
This structure enables us to decompose the nonconvex optimization over multiple networks into an alternating optimization scheme of convex subproblems, leveraging the power of Kronecker product theory \cite{MatrixCookbook}.
In addition, \propose can be extended to stream algorithms, which we call \stream.
In this case, the computation time is independent of the sequence length thanks to incremental network updating.

In summary, our method has the following characteristics:
\begin{itemize}[leftmargin=10pt,nosep]
    \item \textbf{Interpretable}: 
    By estimating mode-specific networks, we prevent networks from becoming excessively large and entangled, which significantly improves the interpretability of modeling results.
    
    \item \textbf{Scalable}:
    Because the network is estimated separately for each mode, the number of nodes and computation time do not increase exponentially. 
    In our experiments, our method is up to 60.5 times faster than existing dynamic network inference methods.
    In addition, \propose can be extended to stream algorithms, making the computational time independent of the sequence length.
   
    \item \textbf{Accurate}:
    Thanks to a model architecture carefully tailored for tensor data, our method accurately captures dependencies between variables. 
    Our experiments demonstrate that our approach improves edge estimation accuracy by up to 73.5\% based on \textit{AUC-ROC}.
\end{itemize}

\begin{figure}[t]
    \centering
    \includegraphics[width=\linewidth]{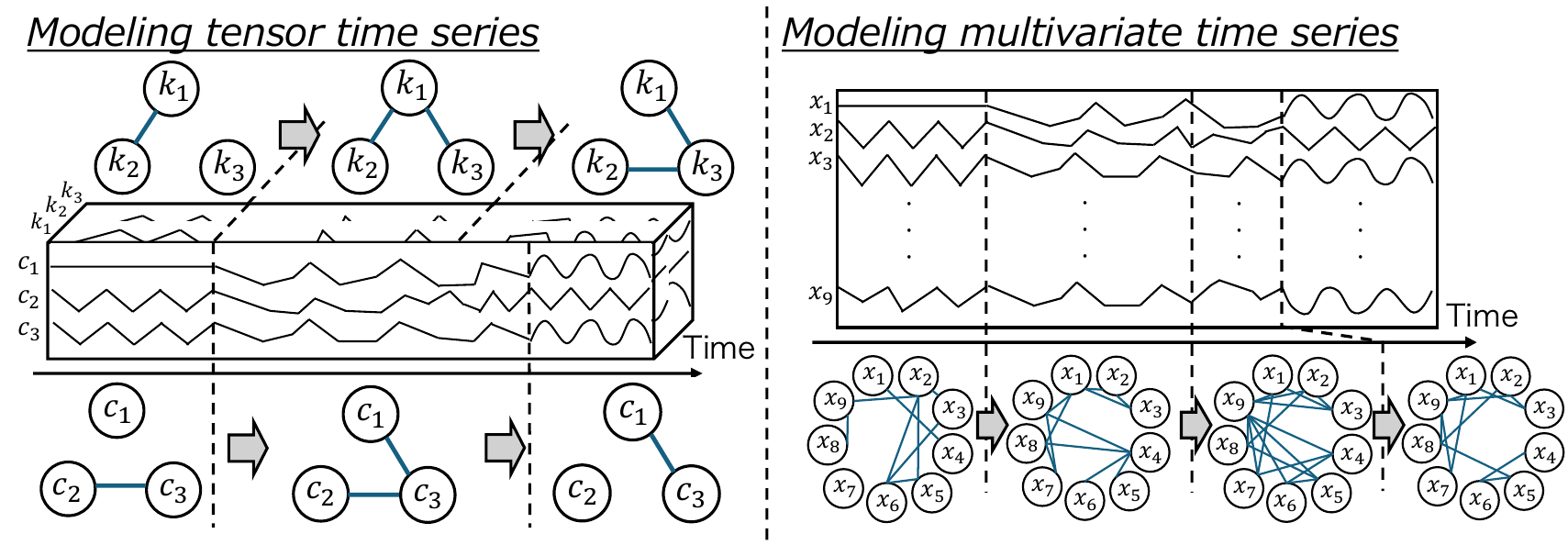}
    \caption{
    Modeling tensor time series (left) versus multivariate time series obtained by flattening tensors (right). 
    Dynamic network inference on flattened data results in large and entangled networks with reduced accuracy and interpretability, whereas estimating mode-specific networks yields clearer and more interpretable results.
    }
    \label{fig:figure1}
\end{figure}


\section{Preliminaries}
\label{sec:Preliminaries}
\mypara{Graphical Lasso}
Graphical Lasso is a statistical method for estimating a graph structure (i.e., network) representing dependencies between variables from multivariate data. 
Specifically, it estimates the Gaussian inverse covariance matrix (i.e., network) $\Theta \in \mathbb{R}^{D \times D}$, also known as the precision matrix, based on the the covariance structure of the data.
By applying $\ell_1$-regularization, it sets many elements to zero, yielding a sparse network.
This allows us to interpret pairwise conditional independence among $D$ variables, for example, if $\Theta_{ij} = 0$, variables $i$ and $j$ are conditionally independent given the values of all the other variables.
Given a series of multivariate input data, the optimization problem is written as follows:
\begin{align}
    \min_{\Theta}  -\ell(\hat{S}, \Theta) + \lambda \|  \Theta \|_{1, \text{od}}, \\
    \ell(\hat{S}, \Theta) = - \mathrm{tr}(\hat{S} \Theta) + \log\det\Theta
\end{align}
where $\Theta$ must be symmetric positive-definite ($\mathrm{S}_{++}^{p}$), $\hat{S}$ is the empirical covariance $\frac{1}{n}\sum_{i=1}^{n}x_{i}x_{i}^{\top}$, $n$ is the number of observations, $\|  \cdot \|_{1, \text{od}}$ indicates the off-diagonal $\ell_1$-norm, and $\lambda \geq 0$ is a hyperparameter for determining the sparsity level of the network. 

\mypara{Time-varying Graphical Lasso (TVGL)}
While Graphical Lasso estimates static networks, TVGL is a method for estimating dynamic networks.
Given a sequence of multivariate observations, it estimates the precision matrix $\Theta_t$ at each time point.
TVGL imposes temporal consistency constraints in addition to sparsity constraints, based on the insight that networks at adjacent time points should share similar structures.
The optimization problem is written as follows:
\begin{align}
\label{eq:obj_TVGL}
    \min_{\Theta_{1}, \dots, \Theta_{T}} \sum_{t=1}^{T} \Bigl\{ -\ell(\hat{S}_{t}, \Theta_t) + \lambda \| \Theta_{t}\|_{1,\text{od}} \Bigr\} + \rho \sum_{t=2}^{T} \psi(\Theta_{t} - \Theta_{t-1}),
\end{align}
where $\hat{S}_{t}$ is the empirical covariance at time $t$, $\psi$ is a convex penalty function that encourages similarity between $\Theta_{t-1}$ and $\Theta_t$, and $\rho \ge 0$ is a hyperparameter that determines how strongly correlated neighboring covariance estimations should be.
The dependencies between variables and their dynamics are represented by the sequence of precision matrices $\{ \Theta_1, \dots, \Theta_T \}$.
Regarding $\psi$ acting as a time-consistency constraint, various options such as the Laplacian penalty, $\ell_1$-penalty, $\ell_2$-penalty, and so on have been proposed, tailored to the domain knowledge of the target data \cite{TVGL}.
This paper employs the Laplacian penalty, which assumes the entire network undergoes gradual changes, and the $\ell_1$-penalty, which assumes only a small number of edges change; however, other choices are certainly possible.
TVGL optimization problem is convex, and convergence to the global optimum is guaranteed.
The optimal network is obtained by alternating direction method of multipliers (ADMM) \cite{ADMM}.
To infer an $d \times d$ network, the dominant computational cost arises from the eigenvalue decomposition required to update $\Theta$, with a time complexity of $O(d^3)$.
Thus, the overall computational complexity of TVGL is $O(Td^3)$.

Many existing time series modeling methods, including TVGL, assume multivariate time series as input. 
Therefore, to apply them to tensor data, the input tensor must be flattened into multivariate time series.
Such transformations not only increase computational complexity due to the increased input dimension but also cause the loss of structural information contained in the original tensor data (e.g., information about regions), significantly reducing the interpretability of the results.
When using TVGL for modeling tensor time series $\X \in \R^{T \times d_1 \times \cdots \times d_M}$, with $D = \Pi_{m} d_m$, the estimated precision matrix has dimensions $D \times D$. 
Considering the computational time is $O(d^3)$, this results in a significant increase in computational cost. 
Furthermore, the complexly entangled network structure leads to low interpretability of the modeling results.
\propose is a method tailored for modeling tensor time series data, avoiding these issues by separately estimating mode-specific networks.

\section{Proposed Method}

In this section, we describe our proposed method, \propose, for modeling tensor time series using multiple mode-specific networks.
First, we define the specific problem and related notation, then explain the model structure, followed by the optimization algorithm.
Additionally, we provide theoretical analysis of the proposed method from several perspectives and describe its extension to stream algorithms.

\subsection{Problem definition and notation}
The main symbols employed in this paper are described in Appendix \ref{appendix:notation}.
Consider a tensor time series $\X \in \R^{T \times d_1 \times \dots \times d_M}$, where $M$ is the number of non-temporal modes, the first mode represents time, and the sequence length is $T$.
We can also rewrite the tensor time series as a sequence of $M^{\text{th}}$-order tensors $\X = \{ \X_{1}, \X_{2}, \dots, \X_{T} \}$, where each $\X_t \in \R^{d_1 \times \cdots \times d_M}$ denotes the observed data at the time step $t$.
As in TVGL, our method can also be applied when multiple observations $(X_{t,1}, \ldots, X_{t,n_t})$ are obtained at the same time step, where $n_t$ denotes the number of observations at time step $t$. 
For notational simplicity, we omit the sample indices unless otherwise required.
In addition, \textit{unfolding} is the process of reordering the elements of a high-order tensor into a matrix.
We define $\text{unfold}(\X_t, m) \in \R^{d_m \times \Dminus}$ as the unfolding process along mode-$m$ of tensor $\X_t$, where $\Dminus = \Pi_{l \neq m} d_l$.
Also, we define $D = \Pi_{m=1}^{M}d_m$.
The problem we address is estimating the time-varying dependency network inherent in the data, similar to TVGL (dynamic network inference).
Here, we consider modeling tensor time series using multiple mode-specific dynamic networks represented in a Kronecker product form (dynamic multi-network inference).
Specifically, we define the network at time $t$ as $\{\Theta_{t}^{(1)}, \ldots, \Theta_{t}^{(M)}\}$ in a multi-network form. 
In this case, when the tensor time series are viewed as a flattened multivariate time series, the overall precision matrix is represented as $K_t = \otimes_{m=1}^{M} \Theta_{t}^{(m)}$.
We formally define our problem as follows.
\begin{problem}[Dynamic Multi-network Inference]
    \textbf{Given} the tensor time series $\X$,
    the goal is to \textbf{estimate} multiple dynamic dependency networks, $\{ \Theta_t \}_{t=1}^{T}$, where each $\Theta_t = \{ \Theta_{t}^{(1)}, \ldots \Theta_{t}^{(M)} \}$.
\end{problem}

\subsection{Multi-mode network}
Since tensor time-series data consists of multiple modes, 
we consider modeling tensor data using multiple networks, 
each of which is a dependency network corresponding to each non-temporal mode.
This prevents the dimensions of the precision matrix from expanding, leading to highly interpretable modeling results.
Inspired by the literature \cite{KGL}, we impose a Kronecker structure on the networks.
Specifically, the modeling result for the entire tensor data is expressed as the Kronecker product of each mode-specific precision matrix, as follows:
\begin{equation}
    \mathrm{vec}(\X_t)\ \sim\ \mathcal N\Bigr(0, (\otimes_{m=1}^M \Theta^{(m)}_t)^{-1}\Bigl)
    \label{eq:dist_vecX}
\end{equation}
This corresponds to assuming that the dependency between two observations is expressed as the product of their respective mode-specific dependencies, 
enabling effective and efficient modeling of tensor time series while leveraging the structural information inherent in the input tensor.
Thanks to this multi-network model structure, compared to standard unstructured models, the number of unknown parameters of \propose is reduced from $O(\Pi_{m=1}^{M} d_{m}^{2})$ to $O(\Sigma_{m=1}^{M} d_{m}^{2})$.
This results in a significant reduction in the computational complexity, improves edge estimation accuracy, and enhances the interpretability of modeling results. 
Our model structure serves as a tool for multifaceted data interpretation from the perspective of each mode, enabling tasks previously unachievable with conventional methods, such as capturing interactions focused solely on specific modes or detecting network change points focused exclusively on specific modes.
Similar to TVGL, each mode-specific network should be sparse and dynamic.

In short, \propose estimates time-varying sparse precision matrices for each non-temporal mode and represents the entire tensor data through its Kronecker product.
The optimization problem is described as follows:
\begin{align}
    \min_{\Theta^{(1)}, \dots , \Theta^{(M)}} \sum_{t=1}^{T} \Bigl\{ l_t(\Theta^{(1)}_t, &\dots , \Theta^{(M)}_t) + \sum_{m=1}^{M} \lambda_m \| \Theta_{t}^{(m)}\|_{1,\text{od}} \Bigr\} \nonumber \\ 
    &+ \sum_{t=2}^{T} \sum_{m=1}^{M} \rho_m \psi(\Theta_{t}^{(m)} - \Theta_{t-1}^{(m)}).
    \label{eq:obj_full}
\end{align}
Here, $l_t(\Theta^{(1)}_t, \dots , \Theta^{(M)}_t)$ is the log-likelihood of $\Theta_t$'s, 
\begin{align}
\label{eq:likelihood_full}
    l_t(\Theta^{(1)}_t, \dots , \Theta^{(M)}_t) = \mathrm{tr}(\hat{S}_{t} K_{t}) - \log\det K_{t},
\end{align}
where each $\Theta_t$ must be symmetric positive-definite ($\mathrm{S}_{++}^{p}$), $\hat{S}_{t}$ is the empirical covariance $\frac{1}{n_t}\sum_{n=1}^{n_t} \mathrm{vec}(X_{t,n}) \mathrm{vec}(X_{t,n})^T$, $n_t$ is the number of observations at time $t$, and $K_{t}$ represents the Kronecker product of all mode-specific networks, that is, $K_t = \otimes_{m=1}^{M}\Theta_{t}^{(m)}$.

\subsection{Optimization algorithm}
Here, we describe an algorithm for estimating time-varying sparse networks.
Optimization problem (\ref{eq:obj_full}) is non-convex and cannot be optimized directly.
Therefore, we propose an alternating optimization algorithm.
Specifically, we alternately optimize the network for one mode while keeping the networks for the other modes fixed.

We introduce two lemmas leading to alternating optimization.

\begin{lemma}
    \label{lamma1}
    We define
    \begin{align}
    \label{eq:ComputeS}
        \hat{S}_{t}^{(m)} = \frac{1}{n_t D_{\setminus m}} \sum_{n=1}^{n_t} A_{t}^{(m)} G_{t}^{(m)} A_{t}^{(m)\top},
    \end{align}
    where $G_{t}^{(m)} = \otimes_{l \neq m} \Theta_{t}^{(l)} \in \mathbb{R}^{D_{\setminus m} \times D_{\setminus m}}$ and $A_{t,n_t}^{(m)} = \text{unfold}(X_{t,n_t}, m) \in \mathbb{R}^{d_m \times D_{\setminus m}}$.
    Then, for any $m$, the following equation holds:
    \begin{equation}
        \mathrm{tr}(\hat{S}_{t} K_{t}) = D_{\setminus m} \mathrm{tr}(\hat{S}_{t}^{(m)} \Theta_{t}^{(m)}).
    \end{equation}
\end{lemma}

\begin{proof}
    Considering $K_{t} = \Theta_{t}^{(m)} \otimes G_{t}^{(m)} \in \mathbb{R}^{D \times D}$ as a $d_m \times d_m$ block matrix along mode $m$, we have $\lbrack K_t \rbrack _{jk} = \Theta_{t,\lbrace jk\rbrace }^{(m)} G_{t}^{(m)} \in  \mathbb{R}^{D_{\setminus m} \times D_{\setminus m}}$.
    
    Here, $\lbrack \hat{S}_{t} K_t \rbrack _{jj} = \sum_{k=1}^{d_m} \lbrack \hat{S}_{t} \rbrack _{jk} \lbrack K_t \rbrack _{kj} = \sum_{k=1}^{d_m} \lbrack \hat{S}_{t} \rbrack _{jk} \Theta_{t, \lbrace kj\rbrace}^{(m)} G_{t}^{(m)}$.
    Therefore, the following equality holds:
    \begin{align}
        \mathrm{tr}(\hat{S}_{t} K_{t}) &= \sum_{j=1}^{d_m} \mathrm{tr}(\lbrack \hat{S}_{t} K_t \rbrack _{jj}) = \sum_{j=1}^{d_m} \sum_{k=1}^{d_m} \Theta_{t, \lbrace kj\rbrace}^{(m)} \mathrm{tr}(\lbrack \hat{S}_{t} \rbrack _{jk} G_{t}^{(m)}) \nonumber \\
        &= \Dminus \sum_{j=1}^{d_m} \sum_{k=1}^{d_m} \Theta_{t, \lbrace kj\rbrace}^{(m)} \hat{S}_{t, \lbrace jk\rbrace}^{(m)} = \Dminus \mathrm{tr}(\hat{S}_{t}^{(m)} \Theta_{t}^{(m)})
    \end{align}
\end{proof}

$\hat{S}_{t}^{(m)}$ can be interpreted as an empirical covariance matrix that aggregates variability along mode $m$, after canceling out the influence of the other modes through whitening with their corresponding precision matrices.
By using this statistic, $\hat{S}_t$ in Eq.(\ref{eq:likelihood_full}) is reduced to a quadratic form with respect to mode $m$, allowing the optimization to be formulated as an alternating procedure: solving TVGL for mode-$m$ while keeping the precision matrices of the other modes fixed, as shown in Eq.(\ref{eq:obj_mode}).
This formulation is made possible by the proposed model structure, which represents the overall network as the Kronecker product of mode-specific networks, thereby decoupling the optimization across modes.
Here, based on the relationship between the Kronecker product and \textit{unfolding} process \cite{Kolda}, and the fact that each $\Theta_{t}^{(m)}$ is a symmetric matrix, the calculation of $\hat{S}_{t}^{(m)}$ in Eq. (\ref{eq:ComputeS}) can be streamlined as follows:
\begin{align}
\label{eq:ComputeS_Y}
    \hat{S}_{t}^{(m)} = \frac{1}{n_t D_{\setminus m}} \sum_{n=1}^{n_t} A_{t}^{(m)} G_{t}^{(m)} A_{t}^{(m)\top} = \frac{1}{n_t D_{\setminus m}} \sum_{n=1}^{n_t} Y_{t}^{(m)} A_{t}^{(m)\top},
\end{align}
where $Y_{t}^{(m)} = \text{unfold}(\X_t \times_{l \neq m} \Theta_{t}^{(l)}, m)$ and $\times_{l}$ denote the mode-$l$ product.
The right-hand side shows that $\hat{S}_{t}^{(m)}$ can be calculated efficiently, avoiding the large intermediate matrix generated by the Kronecker product (i.e., $G_{t}^{(m)}$).

In addition, the following transformation also holds for the second term of Eq. (\ref{eq:likelihood_full}).
\begin{lemma}
    \label{lemma2}
    $\log\det K_t = \sum_{m=1}^{M} D_{\setminus m} \log\det \Theta_{t}^{(m)}$ holds.
\end{lemma}

\begin{proof}
    From \cite{MatrixCookbook}, for square matrices $A \in \mathbb{R}^{a \times a}$ and $B \in \mathbb{R}^{b \times b}$, $\det (A \otimes B) = (\det A)^b (\det B)^a$ holds.
    Applying this recursively yields $\det K_t = \Pi_{m=1}^{M} (\det \Theta_{t}^{(m)})^{D_{\setminus m}}$.
\end{proof}

From the above two lemmas, optimization problem (\ref{eq:obj_full}) can be decomposed into the following subproblems for mode $m = 1, \dots, M$, with the parameters of the other modes held fixed:
\begin{align}
\label{eq:obj_mode}
    \min_{\Theta^{(m)}} \sum_{t=1}^{T} \Bigl\{ \mathrm{tr}(\hat{S}_{t}^{(m)} \Theta_{t}^{(m)}) &- \log\det \Theta_{t}^{(m)} + \frac{\lambda_m}{D_{\setminus m}} \| \Theta_{t}^{(m)}\|_{1,\text{od}}\Bigr\} \nonumber \\
    &+ \frac{\rho_m}{D_{\setminus m}} \sum_{t=2}^{T}  \psi(\Theta_{t}^{(m)} - \Theta_{t-1}^{(m)}).
\end{align}
This is equivalent to the TVGL problem described in Section \ref{sec:Preliminaries} and can be optimized using the TVGL solver.
By iteratively repeating the optimization of the subproblem for all modes, the entire network can be estimated.
Algorithm \ref{alg:ktvgl} shows the overall procedure (i.e., outer loop) of \propose .

\begin{algorithm}[t]
    \small
    \caption{Kronecker Time-Varying Graphical Lasso ($\X $)}
    \label{alg:ktvgl}
\begin{algorithmic}[1]
    \REQUIRE
        Tensor time series $\X$

    \ENSURE
        Dynamic multi-network $\{ \Theta_t \}_{t=1}^{T}$, 
        where $ \Theta_t = \{ \Theta_{t}^{(m)}\}_{m=1}^{M}$ \\
\STATE $\{ \Theta_1 , \dots \Theta_T \} = \textrm{Initialization}(\X)$;
\REPEAT
    \FOR{$m = 1, \ldots, M$}
    \STATE $\{ \hat{S}_{t}^{(m)}\}_{t=1}^{T} \leftarrow \frac{1}{n_t D_{\setminus m}} \sum_{n=1}^{n_t} Y_{t}^{(m)} A_{t}^{(m)\top}$;  /* Eq. (\ref{eq:ComputeS_Y}) */
    \STATE $\{\Theta_{t}^{(m)}\}_{t=1}^{T} \leftarrow \text{TVGL\_solver}(\hat{S}_{t}^{(m)}, \frac{\lambda_m}{\Dminus}, \frac{\rho_m}{\Dminus})$;  /* Eq. (\ref{eq:obj_mode}) */
    \ENDFOR
\UNTIL{convergence;}
\STATE {\bf return} $\{ \Theta_t \}_{t=1}^{T}$;

\end{algorithmic}
\end{algorithm}

\subsubsection{Empirical covariance}
\label{subsubsec:empirical}
In the optimization of network, the log-likelihood (i.e., $\ell(\cdot)$ in Eq. (\ref{eq:obj_TVGL}) and (\ref{eq:likelihood_full})) depends on empirical covariance (i.e., $\hat{S}_{t}$ in Eq. (\ref{eq:obj_TVGL}) and (\ref{eq:likelihood_full})).
$\hat{S}_{t}$ is calculated from the observations obtained at time $t$, but it is quite common for only one observation to be obtained at a single time point.
In such cases, the empirical covariance of TVGL becomes rank-1.
That is, a large covariance matrix must be estimated from just one observation, making the estimation prone to instability.
In contrast, the $\hat{S}_{t}^{(m)}$ computed in Eq. (\ref{eq:ComputeS_Y}) satisfies $\text{rank}(\hat{S}_{t}^{(m)}) \leq \min\{d_m, \Dminus \}$ and can be full-rank, leading to more stable covariance estimation.
This can be interpreted as a result of leveraging the rich structural information contained within the tensor data.
For example, consider the web search volume data with three modes (\textit{time, keyword, country}) discussed in Section \ref{sec:Introduction}.
This corresponds to aggregating information from multiple countries to estimate the network between keywords, and aggregating information from multiple keywords to estimate the network between countries.
In contrast, flattening the data is equivalent to observing only one massive vector of dimension \textit{keyword} $\times$ \textit{country}, causing the empirical covariance to remain rank-1.
This difference significantly impacts the accuracy of modeling, as shown in Section \ref{sec:experiments}.

\subsubsection{Scalability}
\label{subsubsec:scalability}
According to \cite{TVGL}, the overall computational cost of TVGL is $O(Td^3)$.
Therefore, considering that our algorithm can be reduced to alternating optimization of the TVGL subproblem, the computational time for \propose is $O(\#iter \cdot T \Sigma_{m=1}^{M}d_m^{3})$.
Since $\#iter$ is a negligibly small constant value, the total time complexity is $O(T  \Sigma_{m=1}^{M}d_m^{3})$.
On the other hand, when applying TVGL to a flattened tensor, the computational complexity is $O(T\Pi_{m=1}^{M}d_{m}^{3})$. 
In tensor data, the product term $\prod_{m} d_m$ tends to be large, leading to substantial computational cost.
Our method avoids the exponential growth in computation time by leveraging a partitioned multi-network structure.

\subsubsection{Extension to stream processing}
\label{subsubsec:stream}
In scenarios where new data is constantly being observed, it is desirable to update the network incrementally as new data arrives \cite{MAST,RegimeCast}.
Our method can be extended to a stream algorithm using the sliding window approach.
We refer to the extended algorithm as \stream.
Specifically, with window size $w$, dynamic network estimation is performed on the most recent $w$ time steps of data whenever new data is observed.
Here, for the first $w-1$'s $\Theta_t$ values within the window, ADMM updates are performed using the estimated values from the previous window as initial values.
By starting the estimation from reliable initial values, the convergence time for each window is significantly reduced, as shown in Section \ref{sec:experiments}.
Thanks to the sliding window approach, the update time of this algorithm does not depend on the data length, which is an essential property for stream processing.

\section{Experiments}
\label{sec:experiments}
In this section, we run several experiments of \propose on synthetic data, where there are clear ground truth networks.
The experiments were designed to answer the following questions about \propose.
\begin{itemize}
    \item [Q1.] How accurately does it estimate the true network?
    \item[Q2.] How accurately does it detect network change points?
    \item[Q3.] How does it scale in terms of computational time? 
\end{itemize}
Additionally, we evaluate the performance of \stream in streaming scenarios.

\subsection{Synthetic datasets}
We randomly generate synthetic 3rd or 4th-order tensor time series (i.e., $M=2$ or $3$), $\X$, which follow a multivariate normal distribution $\mathrm{vec}(\X_t)\ \sim\ \mathcal N\Bigr(0, (\otimes_{m=1}^M \Theta^{(m)}_t)^{-1}\Bigl)$.
We set the sequence length $T=300$.
For each non-temporal mode, we vary the network structure (i.e., edge positions and counts) at either  $t = 100, 200$,  $t = 100, 250$ or $t = 150, 250$.
The resulting dataset includes both cases where network changes occur at different timings per mode and cases where all networks change simultaneously.
We created multiple datasets while varying the dimensionality per mode.
The generation process for the ground truth network in each mode is based on Erdős-Rényi graph \cite{ErdosRenyi}, specifically as follows:
\begin{enumerate}
    \item [(1)] For $m = 1, \ldots, M$ and $t = 1, \ldots, T$, set $\Theta_{t}^{(m)} \in \R^{d_m \times d_m}$ equal to the adjacency matrix of an Erdos-Renyi directed random graph, where every edge has a $25\%$ chance of being selected.
    \item [(2)] For every selected edge, set the value of that edge to $\sim\ \text{Uniform}([-3.0, 3.0])$. We enforce a symmetry constraint on the network.
    \item [(3)] Let $c$ be the smallest eigenvalue of $\Theta_{t}^{(m)}$, and set $\Theta_{t}^{(m)} = \Theta_{t}^{(m)} + (0.1 + |c|)I$, where $I$ is an identity matrix. This ensures that $\Theta_{t}^{(m)}$ is invertible.
\end{enumerate}
The precision matrix for the entire tensor time series is given by the Kronecker product $K_t = \otimes_{m=1}^{M} \Theta_{t}^{(m)}$.
At each $t$, we observe one sample from true distributions.

\subsection{Experimental setting}
\subsubsection{Evaluation metrics}
We introduce three metrics to measure the accuracy of estimating the true network structure.
\begin{itemize}
    \item \textit{AUC-ROC}: Based on the relationship between the true positive rate and the false positive rate.
    \item \textit{AUC-PR}: Based on precision and recall, and particularly informative in imbalanced settings.
    \item \textit{Best-$F_1$}: The maximum $F_1$ score across all thresholds, reflecting the best trade-off between precision and recall.
\end{itemize}
In addition, based on \cite{TVGL}, we introduce Temporal Deviation Ratio (TDR) to measure how accurately each method could detect the change points in the network structure.
\begin{itemize}
    \item \textit{Temporal Deviation Ratio (TDR)}: We define the temporal deviation as $\| \Theta_{t} - \Theta_{t-1} \|_{F} / \| \Theta_{t} \|_{F}$.
     It indicates how much the network structure changed at each time step.
     TDR is the ratio of the temporal deviation at change points to the average temporal deviation value across all the timestamps.
\end{itemize}

\subsubsection{Baselines}
We compare our approach to two different baselines: the static Kronecker Graphical Lasso (Static KGL) and Time-Varying Graphical Lasso (TVGL).
For the static Kronecker Graphical Lasso, we estimate a static network for each mode that represents the entire dataset.
Additionally, we will perform performance evaluations on \stream.
Hyperparameter settings are described in the Appendix \ref{appendix:hyperparameter}.

\subsection{Results}
\subsubsection{Q1: Accuracy of edge estimation}

\begin{table}[!t]
    \centering
    \caption{Performance comparison using three metrics: AUC-ROC, AUC-PR, Best $F_1$. T.O. denotes timeout under a 5000 seconds limit.}
    \scalebox{0.7}{
    \begin{tabular}{c|c|ccc|ccc|ccc}
        \toprule
         \multicolumn{2}{r|}{Method}& \multicolumn{3}{c|}{{\small \propose }} & \multicolumn{3}{c|}{{\small TVGL}} & \multicolumn{3}{c}{{\small Static KGL}} \\ 
         
        \cmidrule(l{0mm}r{0mm}){1-2}
        \cmidrule(l{1.0mm}r{1.0mm}){3-5}
        \cmidrule(l{1.0mm}r{1.0mm}){6-8}
        \cmidrule(l{1.0mm}r{0mm}){9-11}
        
        $M$ & $d_m$ & {\footnotesize AUCROC} & {\footnotesize AUCPR} & {\footnotesize Best$F_1$} & {\footnotesize AUCROC} & {\footnotesize AUCPR} & {\footnotesize Best$F_1$} & {\footnotesize AUCROC} & {\footnotesize AUCPR} & {\footnotesize Best$F_1$}  \\ 
        \midrule
        
        2 & 3 & \textbf{0.949} & \textbf{0.834} & \textbf{0.759} & $0.779$ & $0.639$ & $0.580$ & $0.892$ & $0.464$ & $0.579$ \\
        
        & 5 & \textbf{0.946} & \textbf{0.837} & \textbf{0.759} & $0.743$ & $0.501$ & $0.494$ & $0.831$ & $0.374$ & $0.468$ \\
        
        & 10 & \textbf{0.955} & \textbf{0.798} & \textbf{0.713} & $0.687$ & $0.355$ & $0.384$ & $0.824$ & $0.347$ & $0.399$ \\

        & 15 & \textbf{0.944} & \textbf{0.721} & \textbf{0.648} & 0.544 & 0.147 & 0.166 & 0.763 & 0.274 & 0.337 \\
        \midrule
        
        3 & 3 & \textbf{0.981} & \textbf{0.875} & \textbf{0.800} & 0.761 & 0.538 & 0.520 & 0.838 & 0.364 & 0.445 \\
        
        & 5 & \textbf{0.988} & \textbf{0.895} & \textbf{0.804} & 0.644 & 0.340 & 0.370 & 0.909 & 0.329 & 0.401\\

        & 10 & \textbf{0.992} & \textbf{0.837} & \textbf{0.754} & T.O. & T.O. & T.O. & 0.843 & 0.219 & 0.290 \\
        
        & 15 & \textbf{0.985} & \textbf{0.699} & \textbf{0.735} & T.O. & T.O. & T.O. & 0.890 & 0.190 & 0.262 \\
        \bottomrule
    \end{tabular}
    }

    \label{tab:accuracy}
\end{table}
Table \ref{tab:accuracy} shows the accuracy of edge estimation for each method.
Thanks to our model structure and algorithm tailored for tensor data, \propose achieves the highest accuracy.
We conducted experiments under a scenario natural for time series analysis, where one observation is obtained at each time step.
In this case, as discussed in Section \ref{subsubsec:empirical}, TVGL calculates the empirical covariance matrix $\hat{S}_{t}$ from a single observation, resulting in $\hat{S}_t$ being rank-1 and leading to low accuracy.
However, KTVGL leverages the structural information of tensor data to compute a full-rank empirical covariance matrix, significantly improving accuracy.
Static KGL can effectively estimate mode-specific networks through its multi-network structure, but its accuracy decreases because it cannot capture temporal changes in the network.
All elements of our method (i.e., the multi-network structure expressed as a Kronecker product and its dynamic estimation—significantly) contribute to improved accuracy.
Overall, our Kronecker-structured dynamic model substantially improves the accuracy of network inference for tensor time series.

\subsubsection{Q2: Accuracy of change point detection}

\begin{table}[!t]
    \centering
    \caption{Performance comparison using Temporal Deviation Ratio (TDR). T.O. denotes timeout under a 5000 seconds limit.}
    \scalebox{0.7}{
    \begin{tabular}{c|cccc|cccc}
        \toprule
        $M$ & \multicolumn{4}{c|}{2} & \multicolumn{4}{c}{3} \\ 
         
        \cmidrule(l{0mm}r{0mm}){1-1}
        \cmidrule(l{1.0mm}r{1.0mm}){2-5}
        \cmidrule(l{1.0mm}r{0mm}){6-9}
        
        $d_m$ &  {\small 3} & {\small 5} & {\small 10} & {\small 15} & {\small 3} & {\small 5} & {\small 10} & {\small 15} \\
        \midrule

        \propose & \textbf{2.72} & \textbf{29.67} & \textbf{75.64} & \textbf{91.37} & \textbf{12.67} & \textbf{113.96} & \textbf{190.96} & \textbf{185.57} \\
        \midrule

        TVGL & 1.96 & 2.28 & 2.24 & 2.20 & 2.36 & 2.35 & T.O. & T.O. \\
        
        \bottomrule
    \end{tabular}
    }

    \label{tab:TDR}
\end{table}

\begin{figure}[t]
   \centering
  \begin{minipage}{0.31\linewidth}
    \centering
    \includegraphics[width=\linewidth]{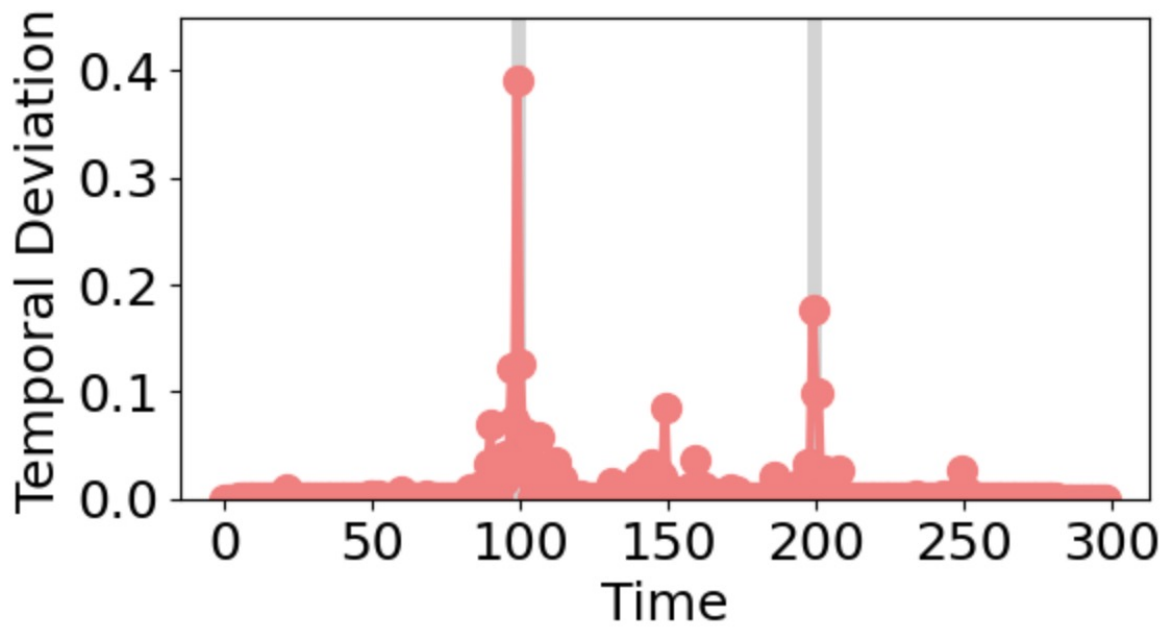}\\
    \includegraphics[width=\linewidth]{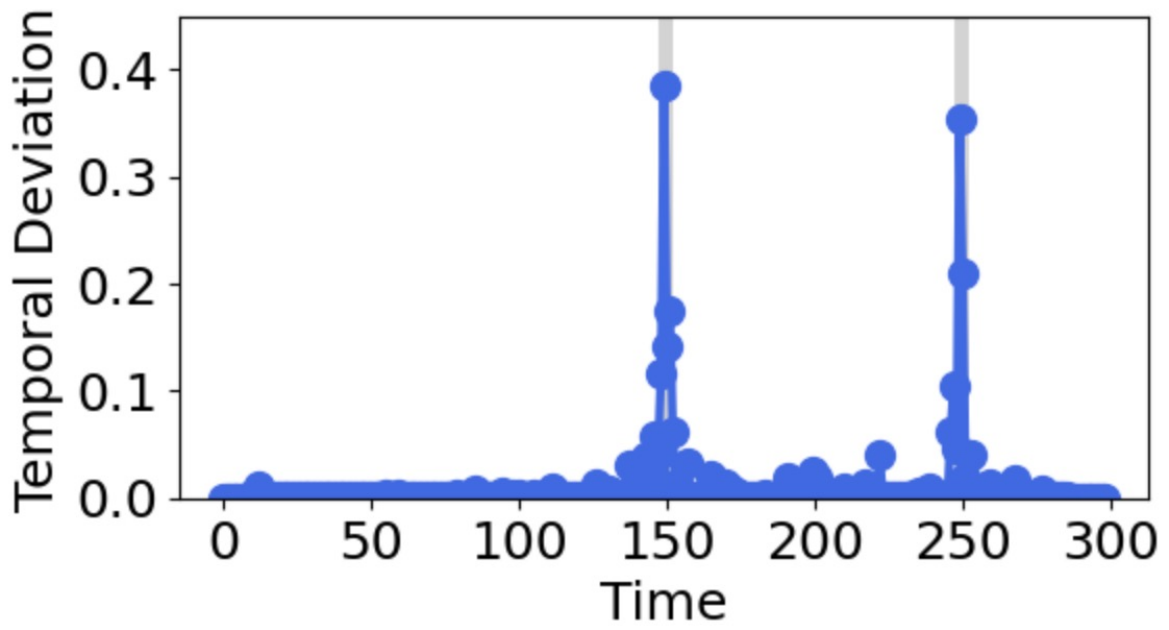}\\
    (a) $M=2$, $d_m = 10$
  \end{minipage}
  \begin{minipage}{0.64\linewidth}
    \centering
    \begin{minipage}{0.48\linewidth}
      \centering
      \includegraphics[width=\linewidth]{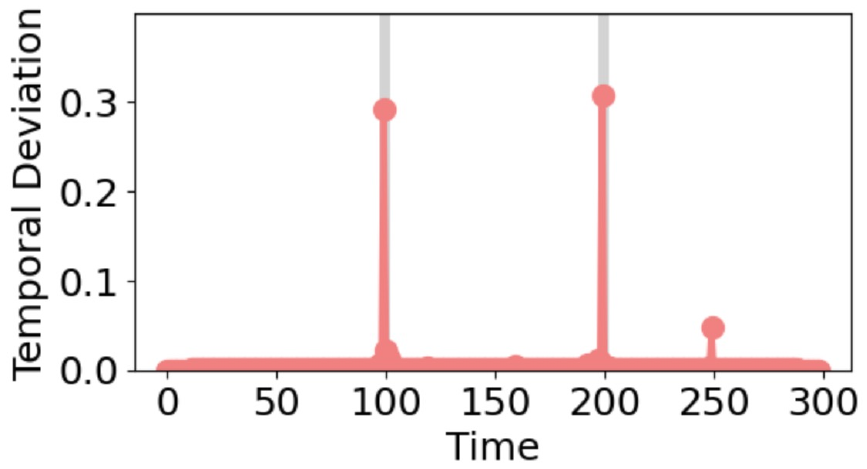}
    \end{minipage}
    \begin{minipage}{0.48\linewidth}
      \centering
      \includegraphics[width=\linewidth]{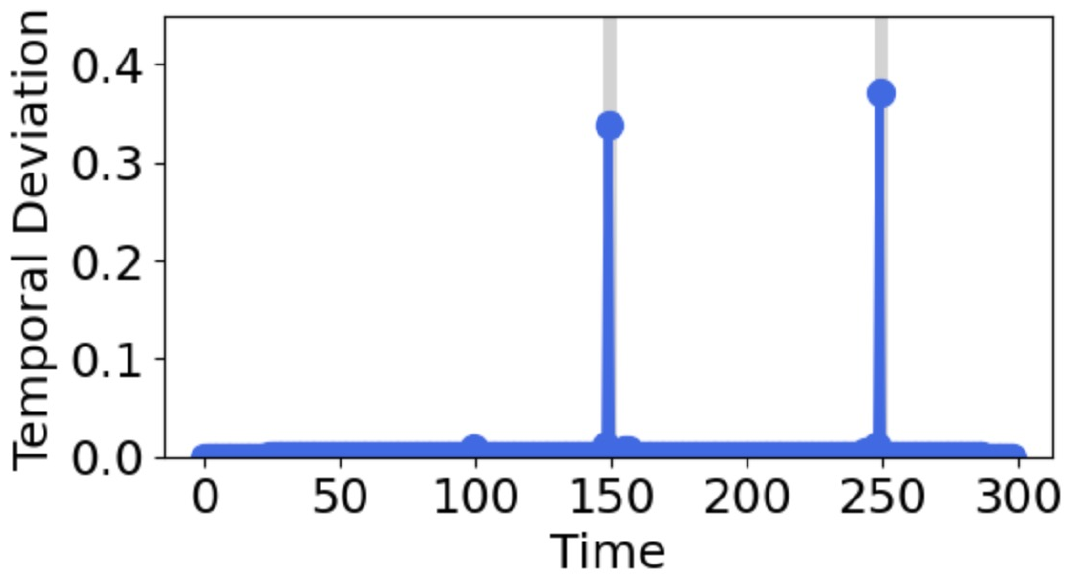}
    \end{minipage}\\
    \begin{minipage}{0.48\linewidth}
      \centering
      \includegraphics[width=\linewidth]{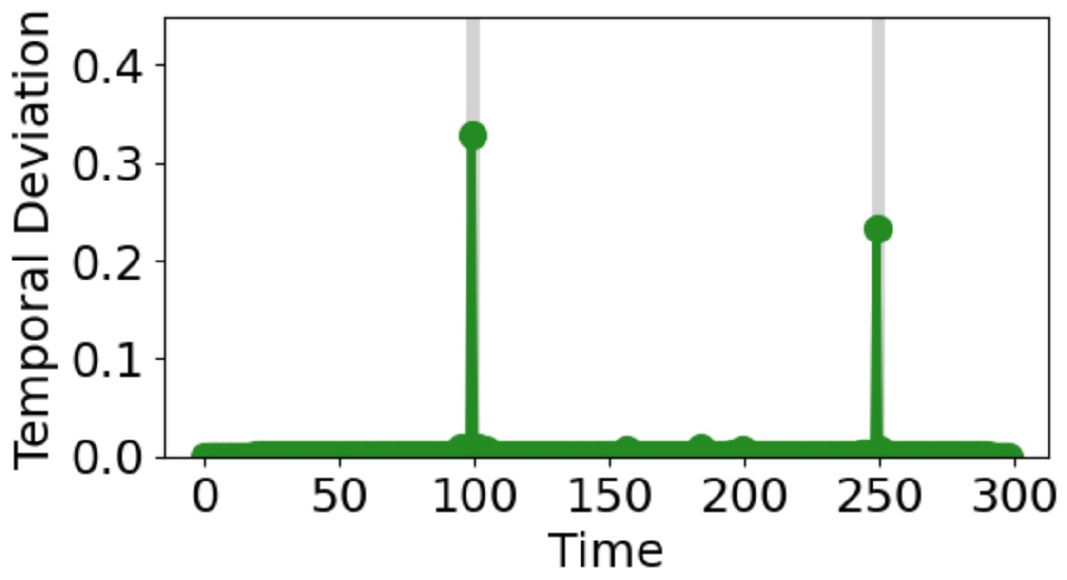}
    \end{minipage}
    \begin{minipage}{0.48\linewidth}
      \centering
      \includegraphics[width=\linewidth]{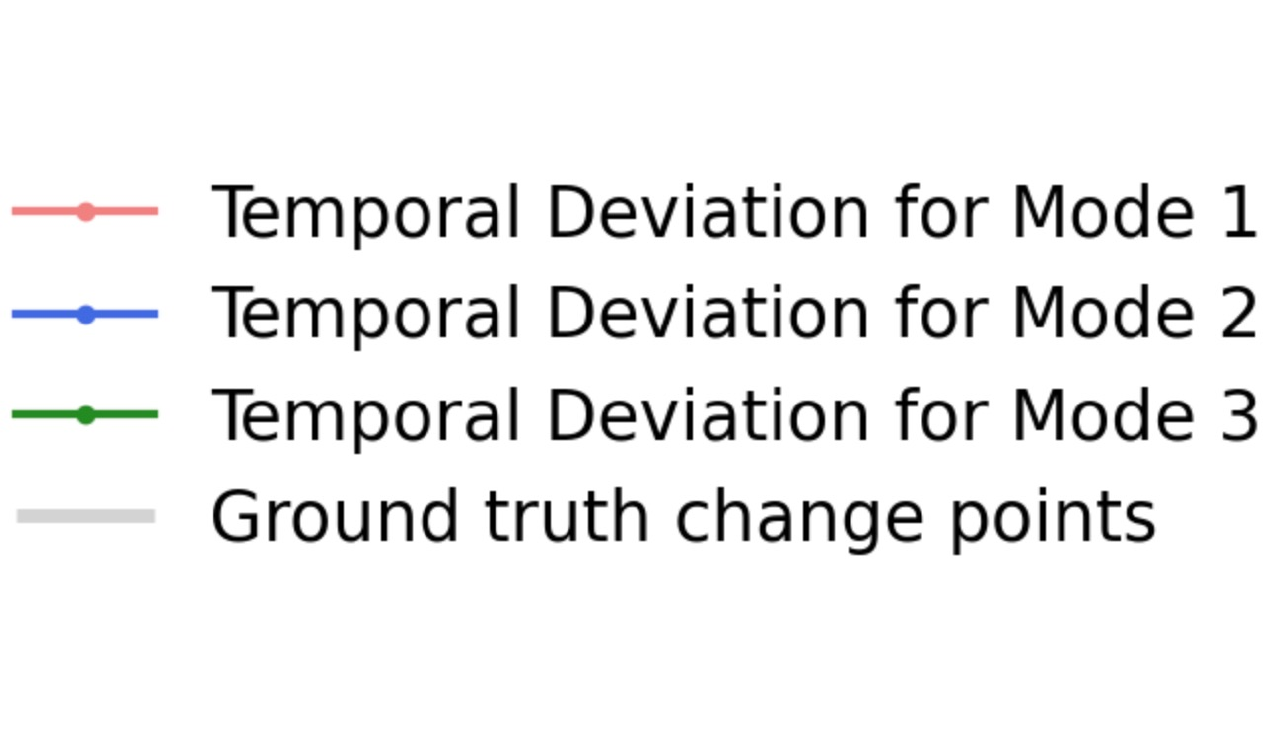}
    \end{minipage}\\
    (b) $M=3$, $d_m = 15$
  \end{minipage}
   
   \caption{Temporal deviation of the estimated networks for datasets with different values of $M$, $d_m$, and ground truth network change points.
   The gray vertical lines indicate the true change points.
   Our method captures these change points across various settings and additionally identifies mode-specific changes.}
   \label{fig:TD}
\end{figure}
We introduce Temporal Deviation Ratio (TDR) as a metric to measure the accuracy of estimating network change points.
Temporal Deviation (TD) is a value indicating how much the network structure changes from one time point to the next. 
TDR is the ratio of the TD at a true change point to the average TD across the entire sequence. 
In other words, a higher TDR indicates that the network change at the true change point is more pronounced relative to other points, reflecting higher confidence in the detected change point.
As shown in Table \ref{tab:TDR}, our method achieves significantly higher TDR scores than the baseline across all datasets.
As the dimensionality of the data increases, more structural information becomes available (see the discussion in Section \ref{subsubsec:empirical}), and thus the accuracy of change point detection tends to be more stable at higher dimensions.
Note that Static KGL is a method for static network inference and is therefore not included in the table.
Additionally, Figure \ref{fig:TD} visualizes TD of the estimated networks and the ground truth change points.
This dataset includes scenarios where multiple networks change simultaneously and scenarios where only one network changes.
Our method accurately estimated the changes in network structure per mode in both scenarios by untangling the seemingly intertwined interactions within the tensor data.
For example, consider the results shown in Figure \ref{fig:TD} (b).
Here, the number of modes $M = 4$, and the ground truth change points are $t^{(1)} = \{100, 200\}$ for mode-1, $t^{(2)} = \{150, 250\}$ for mode-2, and $t^{(3)} = \{100, 250\}$ for mode-3.
In this case, the change points for the entire tensor time series are $t^{(1)} \cup t^{(2)} \cup t^{(3)} = \{ 100, 150, 200, 250\}$, resulting in four change points.
As shown in Figure \ref{fig:TD} (b), \propose correctly distinguishes network changes across modes and accurately identifies which change point originates from which mode.
In other words, our model structure and alternating optimization algorithm leverage the structural information within tensor data, enabling precise mode-specific change point detection that is infeasible for baseline methods.

\subsubsection{Q3: Scalability}
\begin{figure}[t]
   \begin{tabular}{lll}
     \begin{minipage}[b]{0.3\linewidth}
         \centering
         \includegraphics[width=\linewidth]{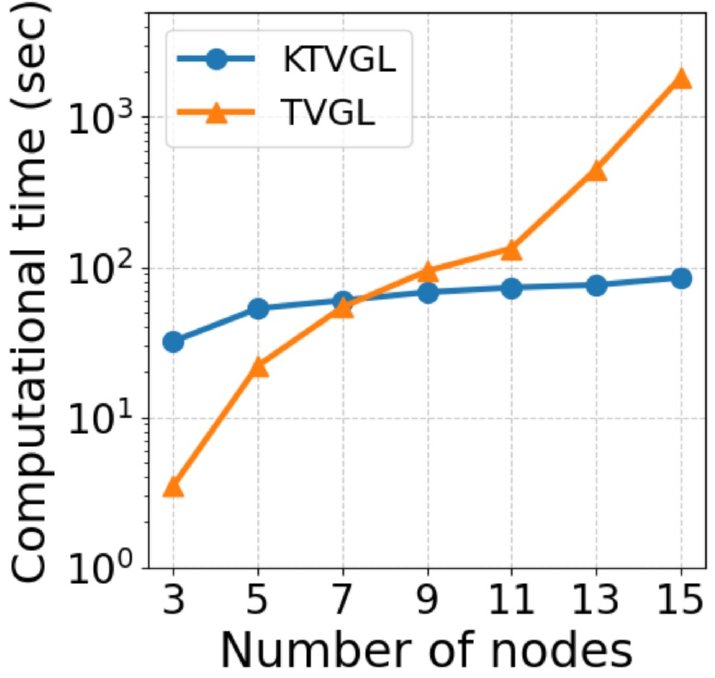}
         (a) vs. $d_m$ ($M=2$)
     \end{minipage}
     &
     \begin{minipage}[b]{0.3\linewidth}
         \centering
         \includegraphics[width=\linewidth]{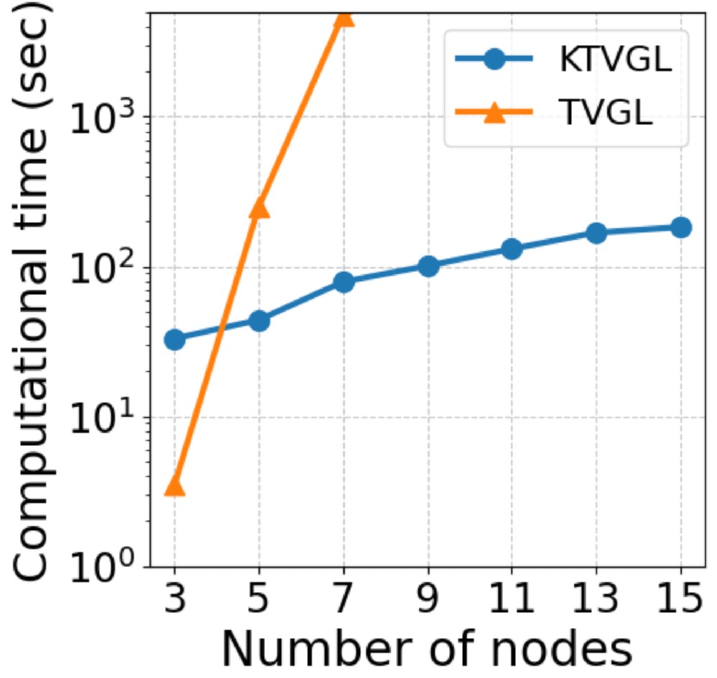}
         (b) vs. $d_m$ ($M=3$)
     \end{minipage}
     &
     \begin{minipage}[b]{0.3\linewidth}
         \centering
         \includegraphics[width=\linewidth]{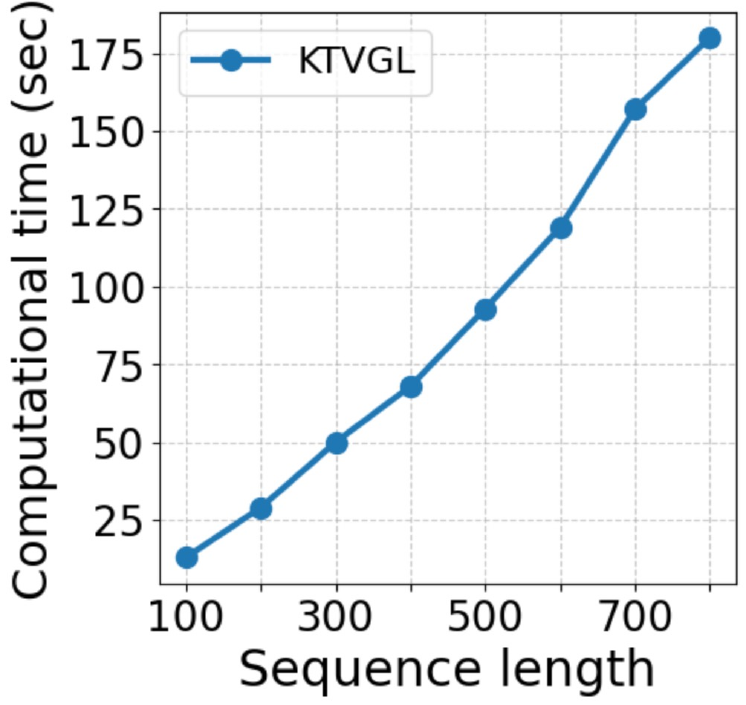}
         (c) vs. $T$
     \end{minipage}
   \end{tabular}
   \caption{Scalability of \propose : 
   (a, b) Wall clock time vs. number of nodes per mode ($d_m$) for number of mode $M = 2,3$. 
   While the computational time of TVGL increases rapidly with $d_m$ and $M$, \propose avoids the exponential growth in computational complexity. 
   (c) Wall clock time vs. sequence length $T$. \propose scales linearly with sequence length.}
   \label{fig:scalability}
\end{figure}
Figure \ref{fig:scalability} shows the wall-clock time of \propose and TVGL.
Figure \ref{fig:scalability} (a) and (b) illustrate the change in computation time for KTVGL and TVGL when varying the number of modes $M$ and dimensions $d_m$.
Note that Static KGL, which performs single static network inference, is very fast (less than 1.0 second), therefore, it is not included in the figure.
TVGL suffers from a significant increase in computation time with growing input data dimensions $d_m$ and number of modes $M$.
This hinders the application of TVGL to tensor time series modeling.
\propose, on the other hand, avoids the exponential increase in computation time due to increases in the number of modes or dimensionality, thanks to its partitioned network structure and efficient alternating optimization framework.
In addition, Figure \ref{fig:scalability} (c) shows the computation time of \propose as the sequence length $T$ varies.
As described in Section \ref{subsubsec:scalability}, \propose scales linearly in terms of the sequence length.

\subsubsection{Performance of stream algorithm}

\begin{figure}[t]
   \begin{tabular}{ll}
     \begin{minipage}[b]{0.46\linewidth}
         \centering
         \includegraphics[width=\linewidth]{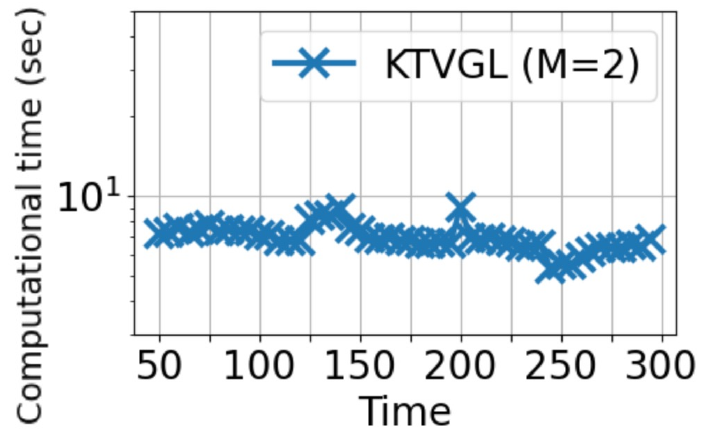}
     \end{minipage}
     &
     \begin{minipage}[b]{0.46\linewidth}
         \centering
         \includegraphics[width=\linewidth]{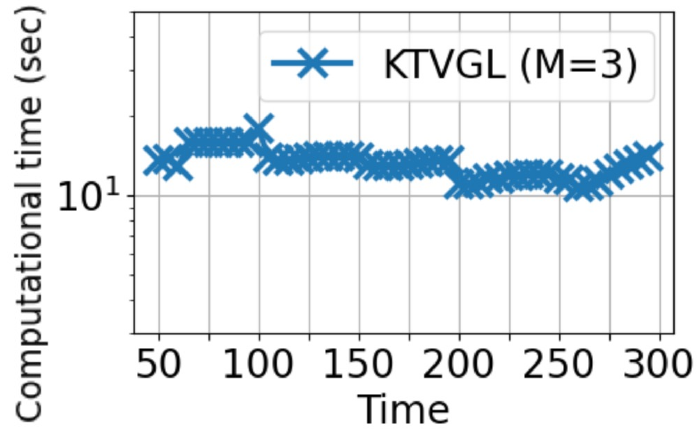}
     \end{minipage}
   \end{tabular}
   \caption{Scalability of \stream: Wall clock time vs. sequence length.
   The computational time at each step is independent of the sequence length.
   We set $d_m=5$, and $M=2$ (left) and $M=3$ (right) respectively.}
   \label{fig:scalability_stream}
\end{figure}

\begin{table}[!t]
    \centering
    \caption{\stream performance.}
    \scalebox{0.8}{
    \begin{tabular}{c|cccc|cccc}
        \toprule
         $M$ & \multicolumn{4}{c|}{2} & \multicolumn{4}{c}{3} \\ 
         \midrule
     $d_m$ & {\small 3} & {\small 5} & {\small 10} & {\small 15} & {\small 3} & {\small 5} & {\small 10} & {\small 15} \\
         \midrule
         \textit{AUC-ROC} & 0.945 & 0.953 & 0.935 & 0.891 & 0.975 & 0.991 & 0.975 & 0.973 \\
         \textit{AUC-PR} & 0.754 & 0.753 & 0.736 & 0.602 & 0.866 & 0.875 & 0.719 & 0.605 \\
         \textit{Best-$F_1$} & 0.698 & 0.680 & 0.660 & 0.578 & 0.767 & 0.784 & 0.651 & 0.564 \\
         \textit{TDR} & 3.42 & 5.83 & 4.59 & 3.04 & 10.73 & 14.19 & 8.81 & 5.66 \\
        \bottomrule
    \end{tabular}
    }

    \label{tab:accuracy_stream}
\end{table}
We also evaluated the performance of \stream, an extension of \propose to a streaming algorithm.
Table \ref{tab:accuracy_stream} shows the accuracy of edge estimation.
By sliding the window, we continuously monitored the latest value within the window and calculated the accuracy between the estimated network and the ground truth network.
\stream exhibited a slight decrease in edge estimation accuracy and a substantial decrease in change point detection performance compared to \propose.
This is due to our continuous monitoring of the latest value within the window.
In the \propose algorithm, the temporal consistency constraint function considers both values before and after the current time. 
However, in the stream algorithm, future values are unobserved, so the current network must be estimated based only on past information.
Nevertheless, its performance remains higher than that of TVGL.
Figure \ref{fig:scalability_stream} shows the computational time of \stream.
The update time of \stream at each step is less than 10 seconds for the case $M=2$ and less than 20 seconds for the case $M=3$, faster than \propose (i.e., Figure \ref{fig:scalability}).
As described in Section \ref{subsubsec:stream}, our stream algorithm uses a sliding window, making computation time independent of sequence length.
Furthermore, thanks to incremental updates, the computation at each step becomes faster.
It is noteworthy that despite its efficiency, \stream still achieves high accuracy in both edge estimation and change point detection.

\section{Case Study}

Many real-world datasets co-evolve through the interaction of multiple series \cite{EcoWeb}.
Estimating these interactions from tensor time series reveals latent structures within the data, providing critical insights for numerous tasks such as marketing and user behavior analysis.
As demonstrated in Section \ref{sec:experiments}, our method achieved outstanding dynamic network inference accuracy in the common time series analysis scenario where one observation is obtained per time point.
In this section, we present several case studies applying KTVGL to analyze real-world tensor time series.

\subsection{Real-world dataset}
Here, we use web search volume data obtained from GoogleTrends \cite{googletrend}.
Specifically, this data represents weekly web search volumes for multiple keywords across multiple regions from 2015 to 2020. 
The data forms a tensor time series consisting of three modes: \textit{(time, keyword, location)}.
This data is considered to reflect the strength of public interest in various keywords, and understanding dependencies between keywords and across regions is crucial for understanding users.
For this study, we selected multiple video services as keywords and prepared two datasets: one representing web search volumes by country (referred to as \textit{country data}) and another representing search volumes by U.S. state (referred to as \textit{region data}).
The data size and included keywords are listed in Table \ref{tab:dataset} in Appnedix \ref{appendix:casestudy}.
We applied moving average smoothing to mitigate the effects of seasonal fluctuations and then normalized all data.

\subsection{Modeling results}
\subsubsection{Country data}

\begin{figure}[t]
    \centering
    \includegraphics[width=\linewidth]{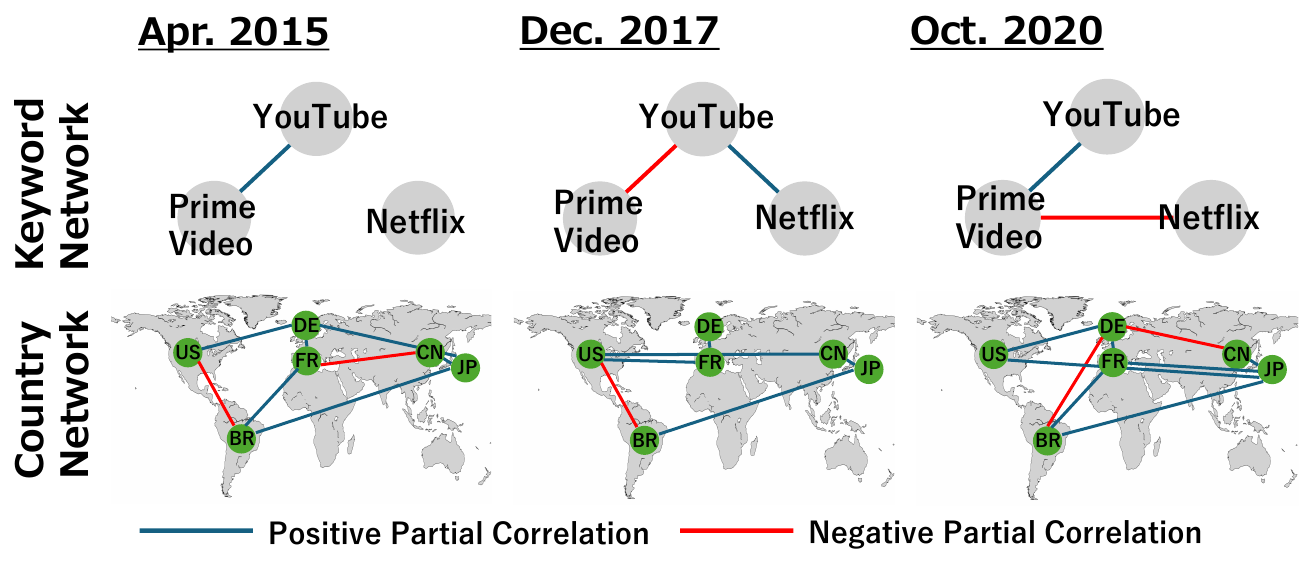}
    \caption{Snapshots of the time-varying networks estimated by KTVGL at three different time points for GoogleTrends \textit{country data}.
    The upper panels represent keyword networks, while the lower panels represent country networks.
    Our multi-network structure provides interpretable modeling results for each non-temporal mode.
    }
    \label{fig:vod}
\end{figure}
Figure \ref{fig:vod} shows modeling results for \textit{country data}.
The upper panel displays the keyword dependency networks, while the lower panel shows the country dependency networks.
The left, center, and right figures present the estimated results for 2015, 2017, and 2020, respectively.
Network edges represent partial correlations between nodes, i.e., the correlation between two variables after filtering out the influence of other variables.
Focusing on the keyword network, ``Prime Video'' and ``YouTube'' showed a positive partial correlation in 2015.
Amazon launched its ``Prime Video'' service globally in 2015.
The observed correlation between “Prime Video” and “YouTube” suggests that, during this period, ``Prime Video'' was also gaining public attention in a manner similar to other major video platforms.
Furthermore, a positive partial correlation was observed between ``YouTube'' and ``Netflix'' in 2017.
By that time, Netflix had already launched its global video-on-demand service.
This correlation suggests that, as the demand for online video services continued to grow, public attention toward ``Netflix'' and ``YouTube'' may have increased concurrently.
In addition, interestingly, a negative correlation was estimated between ``Netflix'' and ``Prime Video'' in 2020.
In 2020, demand for video-on-demand services surged rapidly due to factors like pandemic lockdowns.
During this period, ``Prime Video'' and ``Netflix'' are presumed to have been competing to acquire users.
In other words, a potential competitive relationship existed between these two values: increased attention to one corresponded to decreased interest in the other.
This competitive relationship can be interpreted as a negative correlation estimated after controlling for the effects of other variables.
Focusing on country networks, while networks change over time, countries within the same area---Germany and France, Japan and China---exhibit positive partial correlations across all periods.
We can analyze that neighboring countries in Europe and Asia interact and share related time-series patterns.
Consequently, our method can summarize real-world circumstances as multiple dependent networks corresponding to distinct non-temporal modes, enabling effective interpretation of underlying interactions.

\begin{figure}[t]
    \centering
    \includegraphics[width=\linewidth]{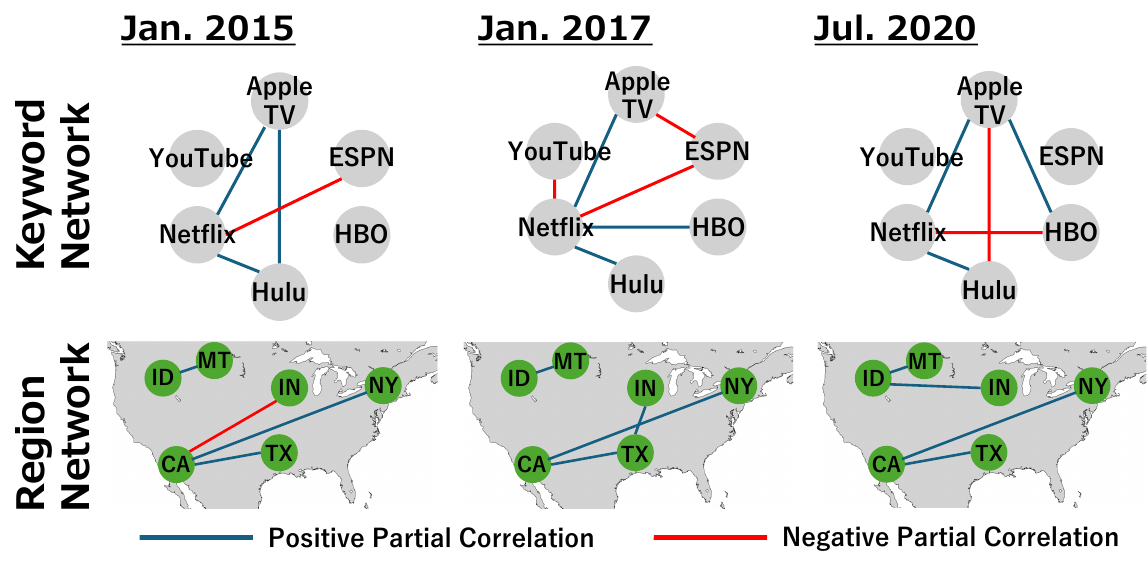}
    \caption{Snapshots of the time-varying networks estimated by KTVGL at three different time points for GoogleTrends \textit{region data}.
    }
    \label{fig:vodRegion}
\end{figure}
\subsubsection{Region data}
Figure \ref{fig:vodRegion} shows modeling results for \textit{region data}.
Focusing on the keyword network for 2015, it indicates that keywords related to subscription-based video-on-demand services---``Netflix'', ``Hulu'', and ``Apple TV''---exhibit positive correlations among each other.
In 2017, while maintaining the relationships observed in the previous snapshot, several new edges were added.
As these edges are primarily connected to ``Netflix'', it can be interpreted as a phase where ``Netflix'' became central, reshaping the competitive landscape of the video service industry.
Furthermore, significant network changes were also observed between 2017 and 2020.
Specifically, the value of the partial correlation between ``Netflix'' and ``HBO'' has shifted from positive to negative.
``HBO'' launched its subscription-based video streaming service in 2020, meaning it became a direct competitor to ``Netflix''.
This negative correlation may reflect competition for user acquisition.
Additionally, Apple launched its new video streaming service in 2019, which may have contributed to the network changes observed here, specifically, the negative correlation between ``Hulu'' and ``Apple TV''.
Additionally, focusing on region networks, positive correlations were observed throughout all periods between California (CA) and Texas (TX), California and New York (NY), and Idaho (ID) and Montana (MT).
This can be interpreted as clusters of urban areas (CA, TX, NY) and rural areas (ID, MT).
These correlations suggest that regions sharing similar socioeconomic environments tend to exhibit comparable temporal trends in public attention.

As shown by the two case studies above, \propose serves as a valuable tool for providing insightful analysis of real-world data through its interpretable multi-network dynamic modeling.

\section{Related Work}

\mypara{Graphical model}
As described in Section \ref{sec:Preliminaries}, graphical lasso \cite{GLasso, GLasso2, GLasso3} is a method for estimating static sparse dependency networks from multivariate data, and numerous extensions and applications have been explored \cite{FinanceGL, BrainGL}.
Extending this framework to multivariate time series data is a natural progression, as latent interactions in real-world data are typically time-varying, thereby motivating the study of dynamic network inference.
Time-Varying Graphical Lasso (TVGL) \cite{TVGL} was proposed as a method specifically designed to estimate time-varying sparse dependency networks. 
It has been recognized for enabling scalable and accurate dynamic network inference, which has in turn inspired a wide range of applications \cite{CovidTVGL, SpaTempTVGL} and further extensions (e.g., Latent-Variable TVGL) \cite{LTVGL, NonparaTVGL, KernelTVGL}.
However, applying TVGL to tensor time series, which are frequently observed in the real world, requires flattening the input tensor into a multivariate time series, leading to large and entangled networks.
To address this limitation, our method directly estimates mode-specific dynamic networks, thereby enhancing both interpretability and scalability.
The idea of imposing a Kronecker structure on graphical lasso is based on \cite{KGL}.
However, it assumes modeling multidimensional time series and cannot be directly applied to tensor time series.
Furthermore, it cannot estimate time-varying networks.
In contrast, our method is designed specifically for modeling tensor time series.
We utilize Kronecker theory and tensor geometry to formulate alternating optimization for tensor data (i.e., Lemma 1 and Eq. (\ref{eq:obj_mode})).
This enables dynamic inference of each network, thereby providing new insights such as capturing interactions focused solely on specific modes and detecting change points focused exclusively on particular modes.
Beyond network inference, graphical lasso has also been applied to various time series tasks, including clustering \cite{TICC, DMM}, missing value imputation \cite{MissNet, MMGL}, and forecasting \cite{TAGM}.
Poisson graphical models for modeling count data have been proposed in addition to Gaussian-based networks \cite{PoissonGM,PoissonGM2,PoissonGM3}.

\mypara{Tensor time series analysis}
Tensor decomposition methods \cite{Kolda, OnlineCP}, including CP and Tucker decompositions, are powerful techniques for analyzing tensor data.
These methods extend matrix decomposition \cite{SVDPCA,NMF} to higher dimensions, extracting latent patterns and structures by decomposing the input tensor into multiple smaller factors.
They are also widely used for analyzing tensor time series, serving not only for pattern extraction \cite{SMF,SSMF,AutoCyclone} but also for anomaly detection \cite{SliceNStitch, ZoomTucker, Dash} and clustering \cite{CubeScope}.
Their application domains are diverse, spanning medical \cite{spartan, dpar2}, finance \cite{Dash}, web data \cite{Dismo,DTracker}, and IoT sensor data \cite{Facets}.
Various extensions have been proposed to accommodate different data properties, such as irregular tensors \cite{Parafac2}, sparsity \cite{spartan}, and non-negativity \cite{NTD}.
Tensor decomposition methods directly decompose data without flattening it into multivariate time series, enabling interpretable summaries that preserve structural information and correlations between modes.
However, while they can capture interactions between modes, they cannot explicitly model correlations between variables within the same mode.
Our method represents a distinct approach from tensor decomposition techniques in that it captures interactions within the same mode.

\section{Conclusion}

In this paper, we propose Kronecker Time-Varying Graphical Lasso (KTVGL), 
designed for modeling tensor time series.
Our approach avoids large, entangled networks by estimating mode-specific networks corresponding to each non-temporal mode. 
This prevents exponential increases in computational complexity and leads to interpretable modeling results.
Our model structure adopts a Kronecker structure for the multi-network, enabling the optimization problem for the multi-network to be decomposed into alternating convex optimization problems for each mode-specific network.
This design allows the model to effectively and efficiently capture dependencies inherent in tensor time series.
In addition, our method can be extended to stream algorithms, making the computational time independent of the sequence length.
Experiments on synthetic data show that our approach improves edge estimation accuracy by up to 73.5\% based on \textit{AUC-ROC} and is up to 60.5 times faster than existing methods.
Furthermore, case studies using real-world datasets demonstrated the effectiveness of our approach.

\section*{acknowledgement} 
This work was partly supported by
JST BOOST, Japan Grant Number JPMJBS2402, 
“Program for Leading Graduate Schools” of the Osaka University, Japan, 
JSPS KAKENHI Grant-in-Aid for Scientific Research Number JP20910053,
JST CREST JPMJCR23M3,
JST START JPMJST2553,
JST CREST JPMJCR20C6,
JST K Program JPMJKP25Y6,
JST COI-NEXT JPMJPF2009,
JST COI-NEXT JPMJPF2115,
the Future Social Value Co-Creation Project - Osaka University,
JSPS KAKENHI Grant-in-Aid for Scientific Research Number JP25K21208. 

\bibliographystyle{ACM-Reference-Format}
\balance
\bibliography{%
BIB/math,
BIB/graphical,
BIB/application,
BIB/tensor
}

\appendix
\section*{Appendix}
\section{Problem definition and notation}
\label{appendix:notation}

\begin{table}[h]
    \centering
    \caption{Symbols and definitions}
    \scalebox{0.77}{
    \begin{tabular}{l|l}
        \toprule
        Symbol & Definition \\ \midrule
        $d_m$ & Number of variables at mode-$m$ \\
        $M$ & Number of non-temporal modes \\
        $T$ & Sequence length \\
        $\X$ & Input tensor time series, i.e., $\X \in \R^{T \times d_1 \times \dots \times d_M}$ \\ 
        $\X_t$ & Observed data at time $t$, i.e., $\X_t \in \R^{d_1 \times \cdots \times d_M}$ and $\X = \{ \X_{1}, \dots, \X_{T} \}$ \\ 
        $D$ & Product of $d_m$ across all modes, i.e., $D = \Pi_{m=1}^{M}d_m$ \\
        $\Dminus$ & $D$ excluding mode-$m$, i.e., $\Dminus = \Pi_{l \neq m} d_l$ \\
        \midrule
        $\Theta_{t}^{(m)}$ & Dependency network for mode-$m$ at time $t$, i.e., $\Theta_{t}^{(m)} \in \R^{d_m \times d_m}$ \\
        $\Theta_t$ & Multi-network at time $t$, i.e., $\Theta_t = \{ \Theta_{t}^{(1)}, \ldots \Theta_{t}^{(M)} \}$ \\
        $K_t$ & Kronecker product of all mode-specific networks, i.e.,  $K_t = \otimes_{m=1}^{M} \Theta_{t}^{(m)}$ \\
        $\hat{S}_{t}$ & Empirical covariance at time $t$ \\
        $\hat{S}_{t}^{(m)}$ & Empirical covariance for mode-$m$ at time $t$ \\
        $\psi$ & Time-consistency constraint function \\
        $\lambda_m$ & Hyperparameter for $\ell_1$ regularization in mode-$m$ \\
        $\rho_m$ & Hyperparameter for time-consistency constraint in mode-$m$ \\
        \bottomrule
    \end{tabular}
    }
    \label{tab:notation}
\end{table}
Table \ref{tab:notation} lists the main symbols that we use throughout this paper.

\section{Experiments}
\subsection{Hyperparameter setting}
\label{appendix:hyperparameter}
For all methods, we estimated the regularization parameters $\lambda$ and $\rho$ via grid search using separately generated evaluation datasets.
Specifically, we set $\{ 0.01, 0.03, 0.05 \}$ as the $\ell_1$-regularization parameter $\lambda$ and $\{ 1.0, 1.5, 2.0 \}$ as the penalty function parameter $\rho$.
For the time-consistency constraint function $\psi$, we adopt the Laplacian penalty for edge estimation.
For change point detection, we adopt the $\ell_1$-penalty, which imposes a strong penalty on temporal changes of networks and is suitable for cases where a small number of edges change significantly.

\section{Case Study}
\subsection{Real-world dataset}
\label{appendix:casestudy}
\begin{table}[h]
    \centering
    \caption{GoogleTrends dataset description}
    \scalebox{0.85}{
    \begin{tabular}{c|c|c}
        \toprule
        Name & Keywords & Data size \\ 
        \midrule
        
        \textit{Video (Country)} & AppleTV/ESPN/HBO/Hulu/Netflix/YouTube & $(313, 6, 6)$ \\
        \midrule
    
        \textit{Video (Region)} & PrimeVideo/Netflix/YouTube & $(313,3,6)$ \\
        \bottomrule
    \end{tabular}
    }
    \label{tab:dataset}
\end{table}
The data size and included keywords are listed in Table \ref{tab:dataset}.
This data represents weekly web search volumes for multiple keywords across multiple regions from 2015 to 2020. 
The data forms a tensor time series consisting of three modes: \textit{(time, keyword, location)}.
\textit{Country} data includes web search volumes for each country, while \textit{region} data includes web search volumes for each state in the United States.

\end{document}